\def\isarxiv{1} %%% for icml submission version, we comment this line

\ifdefined\isarxiv
\documentclass[11pt]{article}

\usepackage[numbers]{natbib}

\else
\documentclass[10pt,twocolumn,letterpaper]{article}

\usepackage[review]{cvpr}  
\input{preamble}

\definecolor{cvprblue}{rgb}{0.21,0.49,0.74}
\usepackage[pagebackref,breaklinks,colorlinks,allcolors=cvprblue]{hyperref}

\fi

\usepackage{amsmath}
\usepackage{amsthm}
\usepackage{amssymb}
\usepackage{algorithm}
\usepackage{algpseudocode}
\usepackage{graphicx}
\usepackage{grffile}
\usepackage{wrapfig,epsfig}
\usepackage{url}
\usepackage{xcolor}
\usepackage{epstopdf}

\usepackage{bbm}
\usepackage{dsfont}

\allowdisplaybreaks

\ifdefined\isarxiv

\usepackage{tikz}
\usepackage{hyperref}  %%% arxiv don't allow this.
\hypersetup{colorlinks=true,citecolor=blue,linkcolor=blue} %%% Zhao : maybe we should comment this in submission.
\usetikzlibrary{arrows}
\usepackage[margin=1in]{geometry}

\else

\usepackage{microtype}

\fi
 
\graphicspath{{./figs/}}

\theoremstyle{plain}
\newtheorem{theorem}{Theorem}[section]
\newtheorem{lemma}[theorem]{Lemma}
\newtheorem{definition}[theorem]{Definition}

\newtheorem{corollary}[theorem]{Corollary}

\newtheorem{fact}[theorem]{Fact}
\newtheorem{remark}[theorem]{Remark}

\newcommand{\R}{\mathbb{R}}

\newcommand{\F}{\mathbb{F}}
\newcommand{\AC}{\mathsf{AC}}
\newcommand{\NC}{\mathsf{NC}}
\newcommand{\TC}{\mathsf{TC}}

\newcommand{\rope}{\mathsf{RoPE}}
\newcommand{\dlogtime}{\mathsf{DLOGTIME}}
\usepackage{mathtools}

\DeclareMathOperator*{\E}{{\mathbb{E}}}
\DeclareMathOperator*{\var}{\mathrm{Var}}

\DeclareMathOperator{\poly}{poly}
\DeclareMathOperator{\round}{round}
\DeclareMathOperator{\ds}{/\!\!/}

\DeclareMathOperator{\diag}{diag}

\makeatletter
\newcommand*{\RN}[1]{\expandafter\@slowromancap\romannumeral #1@}
\makeatother

\usepackage{lineno}
\begin{document}

\ifdefined\isarxiv

\date{}

\title{Circuit Complexity Bounds for RoPE-based Transformer Architecture}
\author{
Bo Chen\thanks{\texttt{ bc7b@mtmail.mtsu.edu}. Middle Tennessee State University.}
\and
Xiaoyu Li\thanks{\texttt{
xli216@stevens.edu}. Stevens Institute of Technology.}
\and
Yingyu Liang\thanks{\texttt{
yingyul@hku.hk}. The University of Hong Kong. \texttt{
yliang@cs.wisc.edu}. University of Wisconsin-Madison.} 
\and 
Jiangxuan Long\thanks{\texttt{ lungchianghsuan@gmail.com}. South China University of Technology.}
\and
Zhenmei Shi\thanks{\texttt{
zhmeishi@cs.wisc.edu}. University of Wisconsin-Madison.}
\and 
Zhao Song\thanks{\texttt{ magic.linuxkde@gmail.com}. The Simons Institute for the Theory of Computing at the University of California, Berkeley.}
}

\else

% \title{Intern Project}
\title{Circuit Complexity Bounds for RoPE-based Transformer Architecture}
\author{author1\\
Institution1\\
Institution1 address\\
{\tt\small firstauthor@i1.org}
\and
author2\\
Institution2\\
First line of institution2 address\\
{\tt\small secondauthor@i2.org}
\and
author3\\
Institution2\\
First line of institution2 address\\
{\tt\small secondauthor@i2.org}
\and
author4\\
Institution3\\
First line of institution3 address\\
{\tt\small secondauthor@i3.org}
}

\fi

\ifdefined\isarxiv
\begin{titlepage}
  \maketitle
  \begin{abstract}
Characterizing the express power of the Transformer architecture is critical to understanding its capacity limits and scaling law. Recent works provide the circuit complexity bounds to Transformer-like architecture. On the other hand, Rotary Position Embedding ($\mathsf{RoPE}$) has emerged as a crucial technique in modern large language models, offering superior performance in capturing positional information compared to traditional position embeddings, which shows great potential in application prospects, particularly for the long context scenario. Empirical evidence also suggests that $\mathsf{RoPE}$-based Transformer architectures demonstrate greater generalization capabilities compared to conventional Transformer models. In this work, we establish a circuit complexity bound for Transformers with $\mathsf{RoPE}$ attention. Our key contribution is that we show that unless $\mathsf{TC}^0 = \mathsf{NC}^1$, a $\mathsf{RoPE}$-based Transformer with $\mathrm{poly}(n)$-precision, $O(1)$ layers, hidden dimension $d \leq O(n)$ cannot solve the Arithmetic formula evaluation problem or the Boolean formula value problem. This result significantly demonstrates the fundamental limitation of the expressivity of the $\mathsf{RoPE}$-based Transformer architecture, although it achieves giant empirical success. Our theoretical result not only establishes the complexity bound but also may instruct further work on the $\mathsf{RoPE}$-based Transformer.

  \end{abstract}
  \thispagestyle{empty}
\end{titlepage}

{\hypersetup{linkcolor=black}
\tableofcontents
}
\newpage

\else
\maketitle
\begin{abstract}
Characterizing the express power of the Transformer architecture is critical to understanding its capacity limits and scaling law. Recent works provide the circuit complexity bounds to Transformer-like architecture. On the other hand, Rotary Position Embedding ($\mathsf{RoPE}$) has emerged as a crucial technique in modern large language models, offering superior performance in capturing positional information compared to traditional position embeddings, which shows great potential in application prospects, particularly for the long context scenario. Empirical evidence also suggests that $\mathsf{RoPE}$-based Transformer architectures demonstrate greater generalization capabilities compared to conventional Transformer models. In this work, we establish a circuit complexity bound for Transformers with $\mathsf{RoPE}$ attention. Our key contribution is that we show that unless $\mathsf{TC}^0 = \mathsf{NC}^1$, a $\mathsf{RoPE}$-based Transformer with $\mathrm{poly}(n)$-precision, $O(1)$ layers, hidden dimension $d \leq O(n)$ cannot solve the Arithmetic formula evaluation problem or the Boolean formula value problem. This result significantly demonstrates the fundamental limitation of the expressivity of the $\mathsf{RoPE}$-based Transformer architecture, although it achieves giant empirical success. Our theoretical result not only establishes the complexity bound but also may instruct further work on the $\mathsf{RoPE}$-based Transformer.
\end{abstract}

\fi

\section{Introduction}
Recently, Large Language Models (LLMs), such as GPT-4~\cite{aaa+23}, Claude~\cite{a24}, Llama~\cite{m24}, and more recently, OpenAI's o1~\cite{o24}, have exhibited remarkable potential to revolutionize numerous facets of daily life, including conversational AI~\cite{lct+24}, AI agents~\cite{xcg+23, cyl+24}, search capabilities~\cite{o24}, and AI assistants~\cite{khc+24, fjl+24}, among others. One of the most significant emergent capabilities of LLMs is their proficiency in handling long-context information, which is essential for effectively processing complex documents such as academic papers, official reports, and legal texts. LLMs also have demonstrated exceptional capabilities in tackling long-context tasks, such as zero-shot summarization~\cite{cam24, zjv+24} and sustaining coherent, extended conversations~\cite{xgw+22, mlt+24}. The o1 model from OpenAI~\cite{o24} represents a major advancement in this field. By leveraging Chain-of-Thought (CoT) reasoning~\cite{wws+22, kgr+22} and incorporating Retrieval Augmented Generation (RAG)~\cite{lpp+20, gxg+23}, it showcases a level of expertise comparable to PhD-level problem solving, with both techniques heavily relying on extensive contextual understanding. 

Large Language Models (LLMs) are primarily built upon the Transformer architecture~\cite{vsp+17}, which uses the self-attention mechanism as its core component. Given this foundational structure, an important question arises: what computational primitives can the components of the Transformer implement, and what problems can the entire system solve collectively? These questions are crucial for interpreting Transformers in a principled manner, understanding the potential limitations of their reasoning capabilities, and fostering trust in deployed Transformer-based systems.

To address the aforementioned questions and to investigate the expressiveness of transformers, prior research has made significant strides. Recent studies, such as \cite{ms23}, have established two key results concerning both non-uniform and $\mathsf{L}$-uniform settings: first, any depth-$ d $ transformer with $ c \log n $-precision can be simulated by a threshold circuit family with constant depth; second, such a transformer can also be simulated by a $\mathsf{L}$-uniform threshold circuit family of constant depth. Further advancing these findings, \cite{ms23_neurips} demonstrate that $\dlogtime$-uniform $\TC^0$ circuits are capable of simulating softmax-attention transformers. Building on this foundation, \cite{chi24} refine these results by increasing the accuracy of approximation. They enhance the precision for softmax-attention transformers from $ O(\log n) $ to $ O(\poly(n)) $, confirming that these transformers fall within the $\dlogtime$-uniform $\TC^0$ class. Additionally, they show that a softmax-attention transformer with an absolute error bound of $ 2^{-O(\poly(n))} $ is also contained within $\dlogtime$-uniform $\TC^0$.

On the other hand, first introduced by \cite{rope24}, Rotary Position Embedding ($\rope$) enhances Transformers by encoding both absolute and relative positional information through a rotation matrix, enabling greater sequence length flexibility, improved attention mechanism efficiency, and better performance on long-text tasks, exemplified by $\rope$-based language models that can summarize an entire book in a single pass. Due to these advantageous properties, $\rope$ has been widely adopted in numerous empirical studies~\cite{cnd+23,bbe+23,bsa+23}. However, despite its considerable success, the underlying mechanisms of $\rope$ remain largely unknown, posing an intriguing mystery in the field. A natural question arises: 
\begin{center}
   {\it Does $\rope$ enhance the expressiveness of the Transformer-based Large Language Model? } 
\end{center}

This work aims to address this question from the perspective of circuit complexity, taking a step forward in theoretically understanding the underlying mechanisms of $\rope$. 
We present a rigorous theoretical investigation of $\rope$-based Transformers, establishing fundamental limits on their computational power. 

Our core approach involved a systematic examination of the circuit complexity for each component of the $\rope$-based architecture, from the basic trigonometric functions to the complete attention mechanism. Ultimately, we prove that these models can be simulated using uniform $\mathsf{TC}^0$ circuits. Furthermore, we show that unless $\mathsf{TC}^0 = \mathsf{NC}^1$, $\rope$-based Transformers with $\poly(n)$-precision, $O(1)$ layers, and a hidden dimension $d \leq O(n)$ are unable to solve either Arithmetic formula evaluation or Boolean formula value problems. This finding is significant because it uncovers fundamental expressivity limitations of $\rope$-based architectures, even though they have shown empirical success in modern language models.

Beyond Merrill and Sabharwal \cite{ms23,ms23_neurips} and Chiang \cite{chi24}, our contribution are summarized as follows:
\begin{itemize}
    \item We prove that unless $\TC^0 = \NC^1$, $\rope$-based Transformer with $\poly(n)$-precision, constant-depth, $\poly(n)$-size can be simulated by a 
    $\dlogtime$-uniform $\TC^0$ circuit family (Theorem~\ref{thm:main_result_tc0}).
    \item We prove that unless $\TC^0 = \NC^1$, a $\rope$-based Transformer with $\poly(n)$-precision, $O(1)$ layers, hidden dimension $d \leq O(n)$ cannot solve the Arithmetic formula evaluation problems (Theorem~\ref{thm:tc_arithmetic}).
    \item We prove that unless $\TC^0 = \NC^1$, a $\rope$-based Transformer with $\poly(n)$-precision, $O(1)$ layers, hidden dimension $d \leq O(n)$ cannot solve the Boolean formula value problem (Theorem~\ref{thm:tc_bool}).
\end{itemize}

\paragraph{Roadmap.}
In Section~\ref{sec:related_work}, we review the related work. In Section~\ref{sec:pre}, we introduce some important computation concepts and Transformer definitions essential for the subsequent sections. In Section~\ref{sec:complexity_each_step}, we give the circuit complexity result of $\rope$-based Transformers. In Section~\ref{sec:hardness}, we give our hardness results. In Section~\ref{sec:conclusion}, we summarizes our theoritical results. 
\section{Related work}\label{sec:related_work}
\paragraph{Complexity and Neural Networks.}
Circuit complexity, a branch of computational complexity theory, studies circuit families as models of computation. Several circuit complexity classes are significant in machine learning. Specifically, $\mathsf{AC}^0$ represents problems highly parallelizable with standard logic gates, while $\mathsf{TC}^0$ extends this to include \textit{threshold gates}, and $\mathsf{NC}^1$ denotes the language recognizable by $O(\log n)$-depth circuits with bounded gate arity \cite{mss22}. It is known that $\mathsf{AC}^0 \subset \mathsf{TC}^0 \subseteq \mathsf{NC}^1$, but whether $\mathsf{TC}^0 \neq \mathsf{NC}^1$ remains an open question. Assuming this inequality, \cite{lag+22} shows that Transformer depth must depend on input sequence length when simulating non-solvable semiautomata.
\cite{llzm24} explore relationships among constant-depth Transformers, Transformers with Chain-of-Thought (CoT), and circuit complexity. They demonstrate: 
$\mathsf{T}[\poly (n), 1, 1] \subseteq  \mathsf{CoT}[\log n, \poly(n), 1, 1] \subseteq \mathsf{AC}^0$ and 
$\mathsf{T}[\poly(n), \log n, 0] \subseteq  \mathsf{CoT}[\log n, \poly(n), \log n, 0] 
    \subseteq  \mathsf{TC}^0$
where $\mathsf{T}[d(n), s(n), e(n)]$ denotes a constant-depth Transformers with embedding size $d(n)$, precision $s(n)$ bits, and exponent bits $e(n)$ for input length $n$ and $\mathsf{CoT}[T(n), d(n), s(n), e(n)]$ denotes a $T(n)$-step CoT of a constant-depth Transformer $\mathsf{T}[d(n), s(n), e(n)]$. Their results provide theoretical insights into the emergent CoT ability of Transformers, showing that intermediate reasoning steps enable tackling more complex problems.
The Strong Exponential Time Hypothesis ({\sf SETH}), introduced by \cite{ip01}, strengthens the $\mathsf{P} \neq \mathsf{NP}$ conjecture by asserting that current best $\mathsf{SAT}$ algorithms are roughly optimal: for every $\epsilon > 0$, there exists $k \geq 3$ such that $k$-$\mathsf{SAT}$ cannot be solved in $O(2^{(1-\epsilon)n})$ time, even randomly. {\sf SETH} is widely used to prove fine-grained lower bounds for various algorithmic problems~\cite{w18} and has been applied to derive lower bounds for Transformer training/inference~\cite{as23, as24_neurips,lss+24} and tensor attention~\cite{as24_iclr24,lssz24_tat}. 
Specifically, \cite{as23} demonstrates that unless the $\mathsf{SETH}$ fails, no algorithm exists that can compute the forward pass of an attention network in truly-subquadratic time. On the other hand, \cite{as24_neurips} establishes that the same condition applies to the backward computation of attention networks, i.e., unless the $\mathsf{SETH}$ fails, no truly-subquadratic time algorithm can be devised for the backward computation of attention networks.
In essence, complexity theory provides a powerful framework for investigating the computational capabilities of neural networks, by rigorously analyzing the computational problems they can efficiently solve.

\paragraph{Limitations of Transformers.}
Transformers have shown exceptional capabilities in natural language processing tasks, yet their effectiveness in mathematical computations remains limited \cite{c22}. Consequently, research efforts have increasingly focused on defining the computational boundaries of Transformers. These studies investigate two types of Transformers: (1) the average-head attention Transformer, where the largest entry in the probability vector is set to 1 and all other entries are set to 0; (2) the softmax-attention Transformer, where the probability vector is produced using a softmax function, formally defined as $\mathsf{Softmax}(X) = \diag(\exp(X) \cdot {\bf 1}_n)^{-1} \cdot \exp(X)$. For the average-head attention Transformer, Merrill, Sabharwal, and Smith \cite{mss22} demonstrate that it can recognize languages beyond the circuit complexity class $\AC^0$ but can be simulated by constant-depth threshold circuits, placing it within the non-uniform $\TC^0$ class. Additionally, \cite{lag+22} prove that softmax-attention Transformers can be simulated by a non-uniform $\TC^0$ circuit. Extending this analysis, \cite{ms23} introduce a generalized similarity function $s: \{0,1\}^p \times \{0,1\}^p \to \{0,1\}^p$, applicable to any similarity function within this mapping, and show that softmax-attention Transformers belong to $\mathsf{L}$-uniform $\TC^0$. Through the conversion of Transformer operations into sentences in $\mathsf{FOM}$ (first-order logic extended to include $\mathsf{MAJORITY}$ quantifiers~\cite{imm98}), \cite{ms23_neurips} demonstrate that $\dlogtime$-uniform $\TC^0$ can simulate softmax-attention Transformers. \cite{chi24} further refine these findings by enhancing approximation accuracy. Specifically, they eliminate error in average-head attention Transformers and improve the precision for softmax-attention Transformers from $O(\log n)$ to $O(\poly(n))$, proving that these Transformers belong to the $\dlogtime$-uniform $\TC^0$ class. Additionally, they show that a softmax-attention Transformer with an absolute error of at most $2^{-O(\poly(n))}$ is also within $\dlogtime$-uniform $\TC^0$. Regarding more practical tasks such as mathematical and decision-making problems, \cite{fzg+23} show that, unless $\mathsf{TC}^0 = \mathsf{NC}^1$, no log-precision Transformer can solve arithmetic and equation-solving problems, nor can any log-precision autoregressive Transformer generate correct answers for the Context-Free Grammar (CFG) Membership Testing problem~\cite{sip96}. These theoretical constraints help explain some of the practical limitations observed when applying Transformers to mathematical tasks. Much of the relevant work in recent years has been related to this such as looped transformer~\cite{as24_rope,lss+24_loop_transformer,lss+24_relu,cls+24}, acceleration~\cite{hyw+23,cll+24,lls+24_prune,lssy24,llss24_sparse,lls+24_conv,smn+24,lls+24_io,lls+24_dp_je,hwsl24,hcl+24,hlsl24,whl+24,xhh+24,hcw+24,szz24,whhl24,hwl24}, graph attention~\cite{vcc+17,wjs+19,bay21,chl+24_gat} and other related works\cite{dswy22_coreset,sy23_des,ssx23_ann,gms23_exp_reg,xsl24,lls+24_grok,lssz24_dp,hsk+24}.

\section{Preliminary}\label{sec:pre}
In this section, we present some preliminary concepts and definitions of our paper. In Section~\ref{sec:pre:notation}, we introduce some basic notations used in our paper. 
In Section~\ref{sec:pre:circuit}, we introduce the basics of the circuit complexity classes.
In Section~\ref{sec:pre:float}, we state the Boolean formula value problem and Arithmetic formula evaluation problem and define some important tools to set up our problem. In Section~\ref{sec:pre:trans_block}, we introduce Rotary Position Embedding ($\rope$) attention and some basic settings in the Transformer.

\subsection{Notations}\label{sec:pre:notation}

For any positive integer $n$, we use $[n]$ to denote set $\{1,2,\cdots, n\}$. We use $\mathbb N := \{0, 1, 2, \ldots\}$ to denote the set of natural numbers. We use $\Pr[]$, $\E[]$, and $\var[]$  to denote the probability, expectation, and variance, respectively. For two vectors $x \in \R^n$ and $y \in \R^n$, we use $\langle x, y \rangle$ to denote the inner product between $x,y$.
We use ${\bf 1}_n$ to denote a length-$n$ vector where all the entries are ones.
We use $X_{i,j}$ to denote the $i$-th row, $j$-th column of $X \in \R^{m \times n}$.
We use $\|A\|_{\infty}$ to denote the $\ell_{\infty}$ norm of a matrix $A \in \R^{n \times d}$, i.e. $\|A\|_{\infty} := \max_{i \in [n], j \in [d]} |A_{i,j}|$. 
For $x_i \in \{ 0,1 \}^ *$, $x_i$ is a binary number of arbitrary length, more generally speaking, $x_i$ is a binary string of length $p$, where each bit is either 0 or 1.

\subsection{Circuit Complexity}\label{sec:pre:circuit}
The Boolean circuit, using $\mathsf{AND}$, $\mathsf{OR}$, and $\mathsf{NOT}$ gates, is a fundamental computational model in computer science, which is formally defined as follows.
\begin{definition}[Boolean circuit, Definition 6.1 on page 102 of~\cite{ab09}]
  A Boolean circuit with $n$ variables is a function $C_n: \{0,1\}^n \to \{0, 1\}$ defined on a directed acyclic graph. The nodes in this graph represent logic gates such as $\mathsf{AND}$, $\mathsf{OR}$, and $\mathsf{NOT}$. Input nodes, which have an in-degree of 0, are assigned one of the $n$ Boolean variables. The circuit evaluates each non-input gate's value by computing the inputs it receives from other gates.
\end{definition}

It is natural to examine the languages that can be recognized by specific families of Boolean circuits since it offers insights into the computational capabilities and efficiency of a certain family of computational models.

\begin{definition}[Languages recognized by a circuit family, Definition 6.2 on page 103 of~\cite{ab09}]
 We say that a language $L \subseteq \{0, 1\}^*$ is recognized by a family $\mathcal C$ of Boolean circuits if for all $x \in \{0,1\}^*$, there exists a Boolean circuit $C_{|x|} \in \mathcal C$ over $|x|$ variables such that $C_{|x|}(x) = 1$ if and only if $x \in L$.
\end{definition}

We now define classes of languages by imposing constraints on the types of logic gates that can be utilized within the circuit families necessary for their recognition. The weakest one we are going to introduce is the $\NC^i$ class.

\begin{definition}[$\mathsf{NC}^i$, Definition 6.21 on page 109 of~\cite{ab09}]
    The class $\mathsf{NC}^i$ consists of languages that can be recognized by Boolean circuits with $O(\poly(n))$ size, $O((\log n)^i)$ depth, and bounded fan-in $\mathsf{AND}$, $\mathsf{OR}$ gates, and $\mathsf{NOT}$ gates.
\end{definition}

When Boolean circuits permit $\mathsf{AND}$ and $\mathsf{OR}$ gates with unbounded fan-in, they gain the capacity to recognize a large class of languages. We define $\AC^i$ class as follows.

\begin{definition}[$\mathsf{AC}^i$, Definition 6.22 on page 109 of~\cite{ab09}]
    The class $\mathsf{AC}^i$ consists of languages that can be recognized by Boolean circuits with $O(\poly(n))$ size, $O((\log n)^i)$ depth, and unbounded fan-in $\mathsf{AND}$, $\mathsf{OR}$ gates, and $\mathsf{NOT}$ gates.
\end{definition}

In fact, $\mathsf{AND}$, $\mathsf{OR}$ gates, and $\mathsf{NOT}$ gates can all be implemented by $\mathsf{MAJORITY}$ gates, where the $\mathsf{MAJORITY}$ gate outputs 0 when half or more arguments are 0 and outputs 1 otherwise. Thus, if we allow Boolean circuits to be equipped with $\mathsf{MAJORITY}$ gates, we get a larger class $\TC^i$. 

\begin{definition}[$\mathsf{TC}^i$, Definition 4.34 on page 126 of~\cite{vol99}]\label{def:tc}
    The class $\mathsf{TC}^i$ consists of languages that can be recognized by Boolean circuits with $O(\poly(n))$ size, $O((\log n)^i)$ depth, and unbounded fan-in $\mathsf{AND}$, $\mathsf{OR}$ gates, $\mathsf{NOT}$ gates, and $\mathsf{MAJORITY}$ gates which can output $1$ when more than half of their inputs are $1$.
\end{definition}

\begin{remark}
    Alternatively, in Definition~\ref{def:tc}, $\mathsf{MAJORITY}$ gates can be replaced by $\mathsf{THRESHOLD}$ or $\mathsf{MOD}$ gates configured around prime values. When a Boolean circuit equipped with any of them, we call it is a threshold circuit.
\end{remark}

Finally, we recall the definition of $\mathsf{P}$ class. 

\begin{definition}[$\mathsf{P}$, Definition 1.20 on page 9 of~\cite{ab09}] The class $\mathsf{P}$ consists of languages that can be recognized by a deterministic Turing machine in polynomial time in input size.    
\end{definition}

 The following fact is a folklore that gives the hierarchy of circuit families.

\begin{fact}[Folklore, page 110 on~\cite{ab09}, Corollary 4.35 on page 126 of~\cite{vol99}]\label{fact:forlore}
   For all $i \in \mathbb N$, $
       \mathsf{NC}^i \subseteq \mathsf{AC}^i \subseteq \mathsf{TC}^i \subseteq \mathsf{NC}^{i+1} \subseteq \mathsf{P}.$
\end{fact}

\begin{remark}
    For $i = 0$, it is known that $\mathsf{NC}^0 \subsetneq \mathsf{AC}^0 \subsetneq \mathsf{TC}^0$. However, whether $\mathsf{TC}^0 \subsetneq \mathsf{NC}^{1}$ is an open problem in circuit complexity. Whether $\mathsf{NC} := \cup_{i\in\mathbb N} \mathsf{NC}^i \subsetneq \mathsf{P}$ is also an open problem. See page 110 in~\cite{ab09}, page 116 in~\cite{vol99} for discussion about these.
\end{remark}

We have defined non-uniform circuit families, which do not need to share structure across varying input sizes and can theoretically handle undecidable problems but are impractical due to their infinite description length. Uniform circuit families offer a more feasible computational model, relevant to complexity and language theory. We first define $\mathsf{L}$-uniformity as follows. 

\begin{definition}[$\mathsf{L}$-uniformity, Definition 6.5 on page 104 of~\cite{ab09}]
    Let $\mathsf{C}$ be a language recognized by a circuit family $\mathcal{C}$ (e.g. $\mathsf{C}$ can be $\mathsf{NC}^i,\mathsf{AC}^i$, or $\mathsf{TC}^i $).
    We say that a language $L \subseteq \{0,1\}^*$ is in $\mathsf{L}$-uniform $\mathsf{C}$ if there exists a Turing machine that, for every $n\in \mathbb N$, maps $1^n$ to a circuit in $\mathcal{C}$ over $n$ variables using $O(\log n)$ space such that $C_n$ recognizes $L$.
\end{definition}

Next, we define $\dlogtime$-uniformity and remark on the relationship between these two different uniformity definitions.
\begin{definition}[$\dlogtime$-uniformity, Definition 4.28 on page 123 of \cite{bi94}]
    Let $\mathsf{C}$ be a language recognized by a circuit family $\mathcal{C}$ (e.g. $\mathsf{C}$ can be $\mathsf{NC}^i,\mathsf{AC}^i$, or $\mathsf{TC}^i$). 
    We say that a language $L \subseteq \{0,1\}^*$ is in $\dlogtime$-uniform $\mathsf{C}$ if there exists a random access Turing machine that, for every $n \in \mathbb N$, maps $1^n$ to a circuit $C_n$ over $n$ variables in $\mathcal{C}$ in $O(\log n)$ time such that $C_n$ recognizes $L$.
\end{definition}

\begin{remark}
    $\dlogtime$-uniformity is equivalent to $\mathsf{L}$-uniformity, with the exception of small circuit complexity classes where the circuits lack the capacity to simulate the machines that create them. See~\cite{bi94, hab02} for more discussion on different notions of uniformity. In this paper, whenever we refer to uniform $\mathsf{TC}^0$, we specifically mean $\dlogtime$-uniform $\mathsf{TC}^0$.
\end{remark}

\subsection{Float Point Numbers}\label{sec:pre:float}
In this section, we introduce some important definitions.
To establish a foundation for our computational framework, we first introduce the essential definitions of floating-point numbers and their operations, which are crucial for implementing Transformer calculations efficiently.

\begin{definition}[Floating-point number, Definition 9 on Page 5 of~\cite{chi24}]
    A $p$-bit floating-point number is a pair $\langle m, e \rangle$ of two integers where the significance $m \in (-2^p,-2^{p-1}] \cup \{0\} \cup [2^{p-1},2^p)$ and the exponent $e \in [-2^p, 2^p)$. The value of the floating point $\langle m, e \rangle$ is the real number $m \cdot 2^e$. We denote the set of all $p$-bits floating-point numbers as $\F_p$.
\end{definition}

To handle these floating-point numbers in practice, we need precise rules for rounding and basic arithmetic operations:

\begin{definition}[Rounding, Definition 9 on page 5 of~\cite{chi24}]
    Let $x$ be a real number or a floating point. We define $\round_p (x)$ as the $p$-bit floating-point number nearest to $x$. When there are two such numbers, we define $\round_p (x)$ as the one with even significance.
\end{definition}

Building on these fundamental definitions, we can now define the core arithmetic operations needed for Transformer computations:

\begin{definition}[Floating-point number operations, page 5 on \cite{chi24}]\label{def:float_operations}
    Let $a,b$ be two integers, we define
    \begin{align*}
        a \ds b := \begin{cases}
            a/b & \mathrm{if~} a/b  \mathrm{~is~a~mutiple~of~} 1/4, \\
            a/b +1/8 & \mathrm{otherwise}.
        \end{cases}
    \end{align*}
  Given two $p$-bits floating points $\langle m_1, e_1\rangle, \langle m_2, e_2\rangle$, we define the following operations:
  \begin{itemize}
      \item addition:
      \ifdefined\isarxiv
        \begin{align*}
         \langle m_1, e_1\rangle + \langle m_2, e_2\rangle := \begin{cases}
          \round_p( \langle m_1+m_2 \ds 2^{e_1-e_2}, e_1 \rangle) & \mathrm{if~} e_1 \geq e_2, \\
          \round_p(\langle m_1 \ds 2^{e_2-e_1} + m_2, e_2 \rangle) & \mathrm{if~} e_1 \leq e_2.
      \end{cases}
      \end{align*}
      \else
    \begin{align*}
          &~ \langle m_1, e_1\rangle + \langle m_2, e_2\rangle \\ := &~ \begin{cases}
          \round_p( \langle m_1+m_2 \ds 2^{e_1-e_2}, e_1 \rangle) & \mathrm{if~} e_1 \geq e_2, \\
          \round_p(\langle m_1 \ds 2^{e_2-e_1} + m_2, e_2 \rangle) & \mathrm{if~} e_1 \leq e_2.
      \end{cases}
      \end{align*}
      \fi

      \item multiplication:
      \begin{align*}
          &~ \langle m_1, e_1\rangle \times \langle m_2, e_2\rangle := \round_p(\langle m_1 m_2, e_1 + e_2 \rangle).
      \end{align*} 

      \item division:
      \ifdefined\isarxiv
    \begin{align*}
           \langle m_1, e_1\rangle \div \langle m_2, e_2\rangle := \round_p(\langle m_1 2^{p-1} \ds m_2, e_1 - e_2 -p + 1\rangle).
      \end{align*}
      \else
\begin{align*}
          &~ \langle m_1, e_1\rangle \div \langle m_2, e_2\rangle \\ := &~ \round_p(\langle m_1 2^{p-1} \ds m_2, e_1 - e_2 -p + 1\rangle)
      \end{align*}
      \fi
      \item comparison:
        \begin{align*}
           \langle m_1, e_1\rangle \leq \langle m_2, e_2\rangle \Leftrightarrow \begin{cases}
          m_1 \leq m_2 \ds 2^{e_1-e_2} & \mathrm{if~} e_1 \geq e_2, \\
          m_1 \ds 2^{e_2 - e_1} \leq m_2 & \mathrm{if~} e_1 \leq e_2.
      \end{cases}
      \end{align*}
      \item floor:
      \begin{align*}
          \lfloor \langle m, e \rangle \rfloor := 
         \begin{cases}
            \langle m 2^{e}, 0 \rangle & \text{if } e \geq 0, \\
            \round ( \langle m / 2^{-e}, 0 \rangle) & \text{if } e < 0.
        \end{cases}
      \end{align*}
  \end{itemize}
\end{definition}

These operations are not just theoretical constructs, they can be efficiently implemented in hardware, as demonstrated by the following lemma:

\begin{lemma}[Standard float point number operations in $\TC^0$, Lemma 10 on page 5 and Lemma 11 on page 6 of~\cite{chi24}]\label{lem:float_operations_TC} 
Let $p$ be a positive integer. If $p \leq \poly(n)$, then the following statements hold:
\begin{itemize}
    \item Part 1. The addition, multiplication, division, and comparison defined in Definition~\ref{def:float_operations} of two $p$-bit floating point numbers is computable by a constant-depth uniform threshold circuit of size $\poly(n)$. We use $d_\mathrm{std}$ to denote the maximum depth needed for these operations.
    \item Part 2. The iterated multiplication of $n$ $p$-bit floating point numbers is computable by a constant-depth uniform threshold circuit of size $\poly(n)$. We use $d_\otimes$ denote the depth needed for the iterated multiplication.
    \item Part 3. The iterated addition of $n$ $p$-bit floating point numbers (rounding after the summation is completed) is computable by a constant-depth uniform threshold circuit of size $\poly(n)$. We use $d_\oplus$ denote the depth needed for the iterated addition.
\end{itemize}
\end{lemma}

\begin{corollary}[Floor operation in $\TC^0$]\label{cor:floor_float_TC}
    Let $p$ be a positive integer. If $p \leq \poly(n)$, then floor operation defined in Definition~\ref{def:float_operations} of a $p$-bit floating point number is computable by a constant-depth uniform threshold circuit of size $\poly(n)$. The maximum depth needed for floor operations is bounded by $d_\mathrm{std}$ in Lemma~\ref{lem:float_operations_TC}.
\end{corollary}
\begin{proof}
    This directly follows from the definition of the floor function in Definition~\ref{def:float_operations}.
\end{proof}

\begin{lemma}[Approximating $\exp$ in $\TC^0$, Lemma 12 on page 7 of~\cite{chi24}]\label{lem:exp}
    If a positive integer $p \leq \poly(n)$, then for every $p$-bit floating point number $x$, there is a constant-depth uniform threshold circuit of size $\poly(n)$ which can compute $\exp(x)$ with a relative error at most $2^{-p}$.  We use $d_{\exp}$ to denote the depth needed for computing $\exp(x)$.
\end{lemma}

\begin{lemma}[Approximating square root in $\TC^0$, Lemma 12 on page 7 of~\cite{chi24}]\label{lem:sqrt}
    If a positive integer $p \leq \poly(n)$, then for every $p$-bit floating point number $x$, there is a constant-depth uniform threshold circuit of size $\poly(n)$ which can compute $\sqrt{x}$ with a relative error at most $2^{-p}$.  We use $d_\mathrm{sqrt}$ to denote the depth needed for computing $\sqrt{x}$.
\end{lemma}

\subsection{Transformer Blocks}\label{sec:pre:trans_block}
With our mathematical foundation established, In this section, we can now describe the key components of Transformer architecture, beginning with the softmax operation that is fundamental to attention mechanisms.

\begin{definition}[Softmax] 
Let $z \in \F_p^{n}$. We define 
$\mathsf{Softmax}: \F_p^{n} \to \F_p^{n}$ satisfying 
\begin{align*}
    \mathsf{Softmax}(z):= \exp(z) / \langle \exp(z) , {\bf 1}_n \rangle.  
\end{align*}
\end{definition}

A key innovation in modern Transformers is the $\rope$, which begins with a basic rotation matrix:

\begin{definition}[Rotation matrix block]\label{def:rotated_matrix}
Let $n$ be the input sequence length, $d$ is given the embedding dimension, $\theta \in \F_p$ , we define the rotation matrix as
\begin{align*}
    R(\theta) := \begin{bmatrix}
        \cos \theta &  -\sin \theta\\
        \sin \theta & \cos \theta
    \end{bmatrix}.
\end{align*}    
\end{definition}

This basic rotation matrix is then extended to handle relative positions in the sequence.

\begin{definition}[Rotation matrix]\label{def:token_position_Rotated_matrix}
Let $j$ be the index of position in the sequence, $i$ the index of tokens, we define the overall relative rotation matrix
\begin{align*}
    R_{j-i}
    = \begin{bmatrix}
        R((j-i) \theta_1) & 0 & \cdots & 0 \\
        0 & R ((j-i) \theta_2) & \cdots & 0 \\
        \vdots & \vdots & \ddots & \vdots \\
       0 & 0 & \cdots & R ((j-i)\theta_{d/2}) \\
    \end{bmatrix}.
\end{align*}
where the angle frequencies $\theta_1, \cdots,\theta_{d/2}$ are a set of given parameters, for details on specifying \( \theta \), see Equation (15) on page 5 of \cite{rope24}.
\end{definition}

Using these rotation matrices, we can define the $\rope$ attention mechanism, which incorporates positional information directly into the attention computation.

\begin{definition}[$\rope$ attention matrix]\label{def:attn_matrix}
Let Rotation matrix $R_{j-i}$ be defined in  Definition~\ref{def:token_position_Rotated_matrix}. Let $W_Q, W_K \in \F_p^{d \times d}$ denote the model weights. Let $X \in \F_{p}^{n \times d}$ denote the representation of the length-$n$ sentence. Then, we define the new attention matrix $A \in \F_{p}^{n \times n}$ by, For $i,j \in [n]$, 
\begin{align*}
    A_{i,j} := & ~\exp( \underbrace{ X_{i,*} }_{1 \times d} \underbrace{ W_Q }_{d \times d} \underbrace{ R_{j-i} }_{d \times d} \underbrace{ W_K^\top }_{d \times d} \underbrace{ X_{j,*}^\top }_{d \times 1}).
\end{align*}
\end{definition}

The attention matrix is then used to compute a single attention layer.

\begin{definition}[Single attention layer]\label{def:single_layer_transformer}
     Let $X \in \F_{p}^{n \times d}$ denote the representation of the length-$n$ sentence. Let $W_V \in \F_p^{d \times d}$ denote the model weights. As in the usual attention mechanism, the final goal is to output an $n \times d$ size matrix where $D := \diag( A {\bf 1}_n) \in \F_{p}^{n \times n}$. Then, we define the $i$-th attention layer $\mathsf{Attn}$ as
\begin{align*}
    \mathsf{Attn}_i (X) := & ~ D^{-1} A X W_V .
\end{align*}
\end{definition}

We combine multiple attention layers with other components to create a complete Transformer architecture.

\begin{definition}[Multi-layer $\rope$-based Transformer] \label{def:multi_layer_self_attn}
Let $m$ denote the number of Transformer layers in the model. 
Let $g_i$ denote components other than self-attention in the $i$-th Transformer layer, where $g_i: \F_p^{n \times d} \to \F_p^{n \times d}$ for any $i \in \{0,1,2,\dots,m\}$. 
Let $\mathsf{Attn}_i$ denote the self-attention module in the $i$-th Transformer layer (see also Definition~\ref{def:single_layer_transformer}). 
Let $X \in \F_p^{n \times d}$ denote the input data matrix.
We define a $m$-layer Transformer $ \mathsf{TF}: \F_p^{n \times d} \to \F_p^{n \times d}$ as 
\begin{align*}
    \mathsf{TF}(X) := g_m \circ \mathsf{Attn}_m \circ \dots \circ g_1 \circ \mathsf{Attn}_1 \circ g_0 (X) ~~ \in \F_{p}^{n \times d},
\end{align*}
where $\circ$ denotes function composition.
\end{definition}

Here we introduce two different kinds of $g_i$ function. First, we introduce the MLP (Multilayer Perceptron) layer.

\begin{definition}[Multilayer Perceptron layer]\label{def:mlp}
    Let $X \in \F_p^{n \times d}$ denote the input data matrix. Let $i\in [n]$. Then, we define the MLP layer as follows:
    \begin{align*}
        g^{\mathrm{MLP}}(X)_{i,*} := \underbrace{W}_{d \times d} \cdot \underbrace{X_{i,*}}_{d \times 1} + \underbrace{b}_{d \times 1}.
    \end{align*}
\end{definition}

Then, we introduce the LN (Layer-wise Normalization) layer:
\begin{definition}[Layer-wise normalization layer]\label{def:layer_norm}
    Let $X \in \F_p^{n \times d}$ denote the input data matrix. Let $i\in [n]$. Then, we define the LN layer as follows:
    \begin{align*}
        g^{\mathrm{LN}} (X)_{i,*}  :=  \frac{X_{i,*} - \mu_i}{\sqrt{\sigma_i^2}},
    \end{align*}
    where $\mu_i := \sum_{j=1}^d X_{i,j} / d$, and $\sigma_i^2 := \sum_{j = 1}^d (X_{i,j} - \mu_i)^2 / d$.
\end{definition}

This multi-layer architecture forms the backbone of modern Transformer models, combining the floating-point operations, attention mechanisms, and positional embeddings defined above into a powerful sequence processing system.

\section{Complexity of \texorpdfstring{$\rope$}{}-based Transformers}\label{sec:complexity_each_step}
In this section, we establish several fundamental results regarding the circuit complexity of basic operations required in Transformer computations. In Section~\ref{sec:approx:trig}, we begin by analyzing trigonometric functions, which are essential for rotary position embeddings. In Section~\ref{sec:compute_matrix_product}, we then proceed to study matrix operations. In Section~\ref{sec:compute_rope_attention_matrix}, we examine the $\rope$-based attention matrix. In Section~\ref{sec:compute_single_RoPE_layer}, we analyze the single $\rope$-Attention layer. In Section~\ref{sec:other_conponents}, we compute some common components other than the self-attention layer. In Section~\ref{sec:conpute_rope_transformer}, we show more details about the complete $\rope$-based Transformer mechanism. In Section~\ref{sec:circuit_complexity_bound}, we show our main results that the circuit complexity bound of $\rope$-based Transformer.

\subsection{Approximating Trigonometric Functions}\label{sec:approx:trig}
In this section, we first demonstrate that basic trigonometric functions, i.e., sine and cosine function, which are fundamental to $\rope$ embeddings, can be computed by threshold circuits.

\begin{lemma}[Trigonometric function approximation in $\TC^0$, informal version of Theorem~\ref{lem:sin_cos:informal} at Appendix~\ref{sec:proof_sincos}]\label{lem:sin_cos}
    If $p \leq \poly(n)$, then for every $p$-bit floating point number $x$, there is a constant-depth uniform threshold circuit of size $\poly(n)$ and depth $8d_\mathrm{std} + d_\oplus + d_\otimes$ which can compute $\sin(x)$ and $\cos(x)$ with a relative error at most $2^{-p}$. To simplify, we use $d_{\triangle}$ denote the depth needed for computing $\sin(x)$ and $\cos(x)$.
\end{lemma}

\subsection{Computing Matrix Products}\label{sec:compute_matrix_product}
In this section, we show that basic matrix multiplication can be computed in $\TC^0$.

\begin{lemma}[Matrix multiplication in $\TC^0$, informal version of Lemma~\ref{lem:matrix_multi:informal} at Appendix~\ref{sec:missing_proof}]\label{lem:matrix_multi}
    Let $A \in \F_p^{n_1 \times d}$, $B \in \F_p^{d \times n_2}$ be two matrices. If $p \leq \poly(n), n_1,n_2 \leq \poly(n),  d \leq n$, then $AB$ can be computable by a uniform threshold circuit with size $\poly(n)$ and depth $(d_\mathrm{std}+d_\oplus)$.
\end{lemma}

\subsection{Computing \texorpdfstring{$\rope$}{}-based Attention Matrix}\label{sec:compute_rope_attention_matrix}
In this section, we extend this to the computation of the attention matrix with positional embeddings, i.e., $\rope$-based attention matrix computation.

\begin{lemma}[$\rope$-based attention matrix computation in $\TC^0$, informal version of Lemma~\ref{lem:matrix_a:informal} at Appendix~\ref{sec:missing_proof}]\label{lem:matrix_a}
    If $p \leq \poly(n)$, then
    the attention matrix $A$ in Definition~\ref{def:attn_matrix}  can be computable by a uniform threshold circuit with size $\poly(n)$ and depth $4(d_\mathrm{std} + d_\oplus) + d_\triangle + d_{\exp}$.
\end{lemma}

\subsection{Computing Single \texorpdfstring{$\rope$}{}-based Attention Layer}\label{sec:compute_single_RoPE_layer}
In this section, we analyze the complete attention layer, the approach allows us to carefully track the circuit depth requirements at each stage.

\begin{lemma}[Single $\rope$-based attention layer computation in $\TC^0$]
    If $p \leq \poly(n)$, then
    the attention layer $\mathsf{Attn}$ in Definition~\ref{def:single_layer_transformer} can be computable by a uniform threshold circuit with size $\poly(n)$ and depth $7(d_\mathrm{std} + d_\oplus) + d_\triangle + d_\mathrm{exp}$.
\end{lemma}\label{lem:attn}
\begin{proof}
    To compute $\mathsf{Attn}$, we need to multiply 4 matrix, namely $D^{-1}, A, X$ and $W_V$. 
    To get these matrices, we need to compute $D$ and $A$. following from $D := \diag( A {\bf 1}_n)$, $D$ can be computed by a depth $d_\oplus$, size $\poly(n)$ uniform threshold circuit following from Part 3. of Lemma~\ref{lem:float_operations_TC}. Following from Lemma~\ref{lem:matrix_a}, computing $A$ needs a circuit of depth $4(d_\mathrm{std} + d_\oplus) + d_\triangle+ d_{\exp}$.
    Then, we can multiply $A, X$ and $W_V$, which can be computed by a depth $2(d_\mathrm{std} + d_\oplus)$, size $\poly(n)$ uniform threshold circuit following from Lemma~\ref{lem:matrix_multi}. 
    Finally, we can compute $D^{-1} \cdot A X W_V$ by apply division in parallel, which can be computed by a depth $d_\mathrm{std}$, size $\poly(n)$ uniform threshold circuit following from Part 1. of Lemma~\ref{lem:float_operations_TC}.
    Combining above circuit, we have
    \begin{align*}
        d_\mathrm{total} = 7(d_\mathrm{std} + d_\oplus) + d_\triangle + d_\mathrm{exp}.
    \end{align*}
    Because the number of parallel operation are $O(\poly(n))$, we can show that $\mathsf{Attn}(X)$ can be computed by a depth $7(d_\mathrm{std} + d_\oplus) + d_\triangle+ d_\mathrm{exp}$, size $\poly(n)$ uniform threshold circuit.

    Thus we complete the proof.
\end{proof}

\subsection{Computing Common Buliding Blocks other than Self-attention layer}\label{sec:other_conponents}
In Definition~\ref{def:multi_layer_self_attn}, we define Multi-layer $\mathsf{RoPE}$-based Transformer with self-attention layer and other components, for example layer-norm and MLP. In this section, we show how to compute these components. 
We first give the circuit complexity for the MLP layer. 
\begin{lemma}[MLP computation in $\mathsf{TC}^0$, informal version of Lemma~\ref{lem:mlp_tc0:informal} at Appendix~\ref{sec:missing_proof}]\label{lem:mlp_tc0}
    If $p \leq \poly(n)$, then
    the MLP layer in Definition~\ref{def:mlp} can be computable by a uniform threshold circuit with size $\poly(n)$ and depth $2d_\mathrm{std} + d_{\oplus}$.
\end{lemma}

Then, we give the circuit complexity for the layer-normalization layer. 
\begin{lemma}[Layer-norm computation in $\mathsf{TC}^0$, informal version of Lemma~\ref{lem:layer_tc0:informal} at Appendix~\ref{sec:missing_proof}]\label{lem:layer_tc0}
    If $p \leq \poly(n)$, then
    the Layer-wise Normalization layer in Definition~\ref{def:layer_norm} can be computable by a uniform threshold circuit with size $\poly(n)$ and depth $5d_\mathrm{std} + 2d_{\oplus} + d_\mathrm{sqrt}$.
\end{lemma}

\subsection{Computing Multi-layer \texorpdfstring{$\rope$}{}-based Transformer}\label{sec:conpute_rope_transformer}
In this section, we show how to compute the multi-layer $\rope$-Transformer.
\begin{lemma}[Multi-layer $\rope$-based Transformer computation in $\TC^0$, informal version of Lemma~\ref{lem:tf:informal} at Appendix~\ref{sec:missing_proof}]\label{lem:tf}
    Suppose that for each $i \in [m]$, $g_i$ in $\mathsf{TF}$ is computable by a constant depth $d_g$ uniform threshold circuit with size $\poly(n)$.
    If $p \leq \poly(n)$, then
    the $\rope$-based Transformer $\mathsf{TF}$ in Definition~\ref{def:multi_layer_self_attn} can be computable by a uniform threshold circuit with size $\poly(n)$ and depth $(m+1)d_g + 7m(d_\mathrm{std} + d_\oplus) + m(d_\triangle+ d_\mathrm{exp})$.
\end{lemma}

\subsection{Main Result: Circuit Complexity Bound of \texorpdfstring{$\rope$}{}-based Transformers}\label{sec:circuit_complexity_bound}
In this section, we are ready to represent our main result. We show the circuit complexity bound of $\rope$-based Transformer. 
\begin{theorem}[Main result, Circuit complexity bound of $\rope$-based Transformers]\label{thm:main_result_tc0}
    Suppose that for each $i \in [m]$, $g_i$ in $\mathsf{TF}$ is computable by a constant depth $d_g$ uniform threshold circuit with size $\poly(n)$. If $p \leq \poly(n), d \leq O(n), m \leq O(1)$, then
    the $\rope$-based Transformer $\mathsf{TF}$ in Definition~\ref{def:multi_layer_self_attn} can be simulated by a uniform $\TC^0$ circuit family. 
\end{theorem}
\begin{proof}
   Since $m = O(1)$, by Lemma~\ref{lem:tf}, the circuit that computes $\mathsf{TF}(X)$ has depth
   \begin{align*}
       (m+1)d_g + 7m(d_\mathrm{std} + d_\oplus) + m(d_\triangle+ d_\mathrm{exp}) = O(1)
   \end{align*}
   and size $\poly(n)$. Therefore it can be simulated by a uniform $\TC^0$ circuit family. 

   Thus we complete the proof.
\end{proof}

In Theorem~\ref{thm:main_result_tc0},
we prove that unless $\TC^0 = \NC^1$, $\rope$-based Transformer with $\poly(n)$-precision, constant-depth, $\poly(n)$-size can be simulated by a $\dlogtime$-uniform $\TC^0$ circuit family. It means that although the $\rope$-based Transformers gain success empirically, it still suffers fundamental expressivity limitations under circuit complexity. We introduce these limitations in the following section. 
\section{Hardness}\label{sec:hardness}
In this section, we state two important problems and the corresponding hardness results. In Section~\ref{sec:arithmetic}, we introduce the Arithmetic formula evaluation problem. In Section~\ref{sec:boolean}, we introduce the Boolean formula value problem. In Section~\ref{sec:hardness_result}, we state our two hardness results.

\subsection{Arithmetic Formula Evaluation Problem}\label{sec:arithmetic}

In this section, we first provide a foundational definition as established in \cite{arithmetic92}.

\begin{definition}[Arithmetic formula, Definition on page 13 of \cite{arithmetic92}]\label{def:arithmetic_formula}
Let $\mathbb{S}$ be a semi-ring (which may also be a ring or field). An arithmetic formula over $\mathbb{S}$ with indeterminates $X_1,X_2, \cdots, X_n$ is defined by:
\begin{itemize}
    \item For $i \in [n]$, $X_i$ is an arithmetic formula.
    \item For every $c \in \mathbb{S}$, c is an arithmetic formula.
    \item If $\alpha$ is an arithmetic formula and $\theta$ is a unary operator of $\mathbb{S}$ then $(\theta \alpha)$ is arithmetic formula.
    \item If $\alpha$ and $\beta$ are arithmetic formulas and $\theta$ is a binary operator of $\mathbb{S}$ then $(\alpha \theta \beta)$ is an arithmetic formula.
\end{itemize}
An arithmetic formula $A$ with indeterminates $X_1, \cdots, X_n$ is denoted by $A(X_1, \cdots, X_n)$.
\end{definition}

Following the definition, we explore its computational implications.

\begin{definition}[Arithmetic formula evaluation problem, Definition on page 14 of \cite{arithmetic92}]\label{def:arithmetic_problem}
Let $\mathbb{S}$ be a ring, field, or semi-ring. The arithmetic formula evaluation problem is: Given an arithmetic formula $A(X_1, X_2, \cdots, X_n)$ over $\mathbb{S}$ and constants $c_1, c_2, \cdots, c_n \in \mathbb{S}$, what is $A(c_1, c_2, \cdots, c_n)$?
\end{definition}

Building upon the previously established definitions, we then establish the computational complexity of the problem.

\begin{lemma}[Theorem 6.1 on page 31 of \cite{arithmetic92}]\label{lem:arithmetic_nc1}
 The arithmetic formula evaluation problem is in $\mathsf{NC}^1$-$\mathsf{complete}$.
\end{lemma}

\subsection{Boolean Formula Value Problem}\label{sec:boolean}

In this section, we now shift our focus to the domain of Boolean formulas and their evaluation.

\begin{definition}[Definition on Page 1 of \cite{b87}]\label{def:pre_a_boolean}
    Let $\Sigma=\{0,1,\land,\lor,\lnot,\mathrm{(},\mathrm{)}\}$, the Boolean formula are given by the following inductive definition:
    \begin{itemize}
        \item $0$ and $1$ are Boolean formulas.
        \item If $\alpha$ and $\beta$ are Boolean formulas, then so are $(\lnot \alpha), (\alpha \land \beta)$ and $(\alpha \lor \beta)$.
    \end{itemize}
\end{definition}

To detail further attributes of these formulas:

\begin{definition}[Definition on page 1 of \cite{b87}]\label{def:pre_b_boolean}
    $|\alpha|$ is the length of $\alpha$, i.e. the number of symbols in the string $\alpha$.
\end{definition}

\begin{definition}[Definition on page 1 of \cite{b87}]\label{def:pre_c_boolean}
    The Boolean formula is defined by the following inductive definition:
    \begin{itemize}
        \item $0$ and $1$ are Boolean formulas.
        \item If $\alpha$ is a Boolean formula then so is $\alpha \lnot$.
        \item If $\alpha$ and $\beta$ are Boolean formulas and if $|\alpha| \geq |\beta|$ then $\alpha \beta \lor$ and $\alpha \beta \land$ are Boolean formulas.
    \end{itemize}
\end{definition}
The Boolean formula is defined in the usual way, where $0$ and $1$ represent $\mathrm{False}$ and $\mathrm{True}$, respectively.
\begin{lemma}[Page 1 on \cite{b87}]\label{lem:boolean_nc1}
The problem of determining the truth value of a Boolean formula is in $\mathsf{NC}^1$-$\mathsf{complete}$.
\end{lemma}

\subsection{Hardness Results}\label{sec:hardness_result}
In this section, we state our two hardness results.
\begin{theorem}\label{thm:tc_arithmetic}
    Unless $\TC^0 = \NC^1$, a $\rope$-based Transformer with $\poly(n)$-precision, $O(1)$ layers, hidden dimension $d \leq O(n)$ cannot solve the Arithmetic formula evaluation problems.
\end{theorem}
\begin{proof}
    This follows from combining Theorem~\ref{thm:main_result_tc0} (circuit complexity bound of $\rope$-base Transformer) and Lemma~\ref{lem:arithmetic_nc1} (the arithmetic formula evaluation problem is in $\mathsf{NC}^1$) which we proved above, and Fact~\ref{fact:forlore} (hierarchy of circuit families). Thus we complete the proof.
\end{proof}

\begin{theorem}\label{thm:tc_bool}
    Unless $\TC^0 = \NC^1$, a $\rope$-based Transformer with $\poly(n)$-precision, $O(1)$ layers, hidden dimension $d \leq O(n)$ cannot solve the Boolean formula value problem. 
\end{theorem}
\begin{proof}
This follows from combining Theorem~\ref{thm:main_result_tc0} (circuit complexity bound of $\rope$-base Transformer) and  Lemma~\ref{lem:boolean_nc1} (the problem of determining the truth value of a Boolean formula is in $\mathsf{NC}^1$) which we proved above, and Fact~\ref{fact:forlore} (hierarchy of circuit families). Thus we complete the proof.
\end{proof}

\section{Conclusion}\label{sec:conclusion}
In this work, we provide a rigorous theoretical analysis of $\rope$-based Transformers, establishing fundamental bounds on their computational capabilities. Our main idea was to systematically analyze the circuit complexity of each component in the $\rope$-based architecture, from basic trigonometric functions to the complete attention mechanism, ultimately proving that these models can be simulated by uniform $\mathsf{TC}^0$ circuits. More importantly, we demonstrate that unless $\mathsf{TC}^0 = \mathsf{NC}^1$, $\rope$-based Transformers with $\mathrm{poly(n)}$-precision, $O(1)$ layers, and hidden dimension $d \leq O(n)$ cannot solve either arithmetic formula evaluation or Boolean formula value problems. This result is particularly significant as it reveals fundamental limitations in the expressivity of $\rope$-based architectures, despite their empirical success in modern language models.

\section{Discussion}
One limitation is that our analysis focuses primarily on the forward computation aspects, which and assumes constant-depth nonlinear activation functions, leaving open questions about training dynamics and the impact of more complex activation functions. It would be interesting to extend our theoretical results to analyze other variants of positional embeddings and investigate whether similar complexity bounds hold for more sophisticated Transformer architectures. Furthermore, our results suggest a potential gap between theoretical limitations and empirical performance.

\ifdefined\isarxiv
%\section*{Acknowledgments}
\bibliographystyle{alpha}
\bibliography{ref}
\else
{
\small
\bibliography{ref}
\bibliographystyle{ieeenat_fullname}
}

\fi

\newpage
\onecolumn
\appendix
\begin{center}
    \textbf{\LARGE Appendix}
\end{center}
\section{Proof of Lemma~\ref{lem:sin_cos}}\label{sec:proof_sincos}
Here we present the proof of Lemma~\ref{lem:sin_cos}. We restate Lemma~\ref{lem:sin_cos} below
\begin{lemma}[Trigonometric function approximation in $\TC^0$, formal version of Lemma~\ref{lem:sin_cos}]\label{lem:sin_cos:informal}
    If $p \leq \poly(n)$, then for every $p$-bit floating point number $x$, there is a constant-depth uniform threshold circuit of size $\poly(n)$ which can compute $\sin(x)$ and $\cos(x)$ with a relative error at most $2^{-p}$. We use $d_{\triangle}$ denote the maximum depth needed for computing $\sin(x)$ and $\cos(x)$.
\end{lemma}
\begin{proof}
    For $\sin(x)$ where $x \in \F_{p}$, we can define:
    % \begin{align*}
    $
        k := \left\lfloor \frac{x}{2/\pi} \right\rfloor
    $ 
    % \end{align*}
    and
    \begin{align*}
        r := \begin{cases}
            x - k\pi / 2 & \mathrm{if~} x - k\pi / 2 \leq \pi / 4,\\
            (k+1)\pi / 2 - x & \mathrm{else}.
        \end{cases}
    \end{align*}
    
    Using truncated Taylor series of $\sin(r)$, we have:
    \begin{align*}
        \sin(r) = \sum_{i = 0}^{N - 1} (-1)^i \frac{r^{2i + 1}}{(2i + 1)!} + R_N^{\sin}(r)
    \end{align*}
    For $R_N^{\sin}(r)$, we can show:
    \begin{align*}
        R_N^{\sin}(r) 
        \leq & ~ (\pi/4)^{2N + 1} \frac{1}{(2N+1)!}\\
        \leq & ~ \frac{1}{(2N+1)!}\\
        = & ~ O(1/N!) \\
        \leq & ~ O(2^{-N})
    \end{align*}
    where the first step follows from the definition of the Lagrange remainder term,
    the second step follows from $(\pi/4)^{2N + 1} \leq 1$,
    the fourth step follows from $O(2^x) < O(x!)$ holds for any positive $x$.

    Similarly, using truncated Taylor series of $\cos(r)$, we have:
    \begin{align*}
        \cos(r) = \sum_{i = 0}^{N - 1} (-1)^i \frac{r^{2i}}{(2i)!} + R_N^{\cos}(r)
    \end{align*}
    For $R_N^{\cos}(r)$, we can show:
    \begin{align*}
        R_N^{\cos}(r) 
        \leq & ~ (\pi/4)^{2N} \frac{1}{(2N)!}\\
        \leq & ~ \frac{1}{(2N)!}\\
        = & ~ O(1/N!) \\
        \leq & ~ O(2^{-N})
    \end{align*}
    where the first step follows from the definition of the Lagrange remainder term,
    the second step follows from $(\pi/4)^{2N + 1} \leq 1$,
    the fourth step follows from $O(2^x) < O(x!)$ holds for any positive $x$.
    Then, we have 
    \begin{align*}
        \sin (x) = 
        \begin{cases}
            \sin(r) & \mathrm{if~} x - k\pi / 2 \leq \pi / 4,\\
            \cos(r) & \mathrm{else}.
        \end{cases}
    \end{align*}
    and
    \begin{align*}
        \cos (x) = 
        \begin{cases}
            \cos(r) & \mathrm{if~} x - k\pi / 2 \leq \pi / 4,\\
            \sin(r) & \mathrm{else}.
        \end{cases}
    \end{align*}

    Because of similar calculation step between $\sin(x)$ or $\cos(x)$, we can show the depth of circuit to compute them following from Lemma~\ref{lem:float_operations_TC} and Corollary~\ref{cor:floor_float_TC}:
    \begin{enumerate}
        \item To get the value of $k$, we need to calculate floor and division, which use depth-$2d_\mathrm{std}$ circuit.
        \item To get the value of $r$, we need to calculate addition, comparison, multiplication and division, which use depth-$4d_\mathrm{std}$ circuit.
        \item To get the value of $\sin(r)$ or $\cos(r)$, we need to calculate addition and iterated addition. For each entry in iterated addition, we need to calculate multiplication, division and iterated multiplication in parallel, which use depth-$(3d_\mathrm{std} + d_\otimes + d_\oplus)$ circuit.
        \item To get the value of $\sin(x)$ or $\cos(x)$, we need to calculate comparison, which use depth-$d_\mathrm{std}$ circuit.
    \end{enumerate}
    Finally, we can show
    \begin{align*}
        d_{\triangle} = 8d_\mathrm{std} + d_\oplus + d_\otimes.
    \end{align*}
    Thus we complete the proof.
\end{proof}

\section{Mssing proofs in Section~\ref{sec:complexity_each_step}}\label{sec:missing_proof}
Here we present some missing proofs in Section~\ref{sec:complexity_each_step}. First we show the proof of Lemma~\ref{lem:matrix_multi} below.
\begin{lemma}[Matrix multiplication in $\TC^0$, formal version of Lemma~\ref{lem:matrix_multi}]\label{lem:matrix_multi:informal}
    Let $A \in \F_p^{n_1 \times d}$, $B \in \F_p^{d \times n_2}$ be two matrices. If $p \leq \poly(n), n_1,n_2 \leq \poly(n),  d \leq n$, then $AB$ can be computable by a uniform threshold circuit with size $\poly(n)$ and depth $(d_\mathrm{std}+d_\oplus)$.
\end{lemma}
\begin{proof}
    For each $i \in [n_1]$ and $j \in [n_2]$, the entry $(AB)_{i,j}$ is given by $
        (AB)_{i,j} = \sum_{k=1}^{d} A_{i,k} B_{k,j}. $
    By Part 1 of Lemma~\ref{lem:float_operations_TC}, each product $A_{i,k} B_{k,j}$ can be computed by a uniform threshold circuit of depth $d_{\mathrm{std}}$. Since these products for different $k$ can be computed in parallel, all products $A_{i,k} B_{k,j}$ for $k \in [d]$ can be computed simultaneously in depth $d_{\mathrm{std}}$.

    Next, by Part 3 of Lemma~\ref{lem:float_operations_TC}, the sum $
        \sum_{k=1}^{d} A_{i,k} B_{k,j} $
    can be computed by a uniform threshold circuit of depth $d_{\oplus}$. Therefore, the total depth required to compute $(AB)_{i,j}$ is $d_{\mathrm{std}} + d_{\oplus}$. Since we can compute $(AB)_{i,j}$ for all $i \in [n_1]$ and $j \in [n_2]$ in parallel, the overall depth of the circuit remains $d_{\mathrm{std}} + d_{\oplus}$. The size of the circuit is polynomial in $n$ because $n_1, n_2, d \leq \mathrm{poly}(n)$, and each operation is computed by a circuit of polynomial size. Therefore, $AB$ can be computed by a uniform threshold circuit with size $\mathrm{poly}(n)$ and depth $d_{\mathrm{std}} + d_{\oplus}$. Thus we complete the proof.
\end{proof}

Here we state the proof of Lemma~\ref{lem:matrix_a}.

\begin{lemma}[$\rope$-based attention matrix computation in $\TC^0$, formal version of Lemma~\ref{lem:matrix_a}]\label{lem:matrix_a:informal}
    If $p \leq \poly(n)$, then
    the attention matrix $A$ in Definition~\ref{def:attn_matrix}  can be computable by a uniform threshold circuit with size $\poly(n)$ and depth $4(d_\mathrm{std} + d_\oplus) + d_\triangle + d_{\exp}$.
\end{lemma}
\begin{proof}
    For each $i, j \in [n]$, we need to compute the entry $A_{i,j}$ of the attention matrix $A$ as defined in Definition~\ref{def:attn_matrix}.

    By Lemma~\ref{lem:sin_cos}, each entry of $R_{j-i}$ can be computed using a uniform threshold circuit of size $\mathrm{poly}(n)$ and depth $d_\triangle$. Since $n \leq \mathrm{poly}(n)$, all entries of $R_{j-i}$ can be computed in parallel with the same circuit size and depth.

    Using Lemma~\ref{lem:matrix_multi}, the matrix product $W_Q R_{j-i}$ can be computed by a uniform threshold circuit of size $\mathrm{poly}(n)$ and depth $d_{\mathrm{std}} + d_\oplus$.

    Applying Lemma~\ref{lem:matrix_multi} again, the product $(W_Q R_{j-i}) W_K^\top$ can be computed with the same circuit size and depth $d_{\mathrm{std}} + d_\oplus$.

    Next, the scalar product
    \begin{align*}
        s_{i,j} = X_{i,*} (W_Q R_{j-i} W_K^\top) X_{j,*}^\top
    \end{align*}
    can be computed using a uniform threshold circuit of size $\mathrm{poly}(n)$ and depth $2(d_{\mathrm{std}} + d_\oplus)$, again by Lemma~\ref{lem:matrix_multi}.

    Using Lemma~\ref{lem:exp}, the exponential function $A_{i,j} = \exp(s_{i,j})$ can be computed by a uniform threshold circuit of size $\mathrm{poly}(n)$ and depth $d_{\exp}$.

    Combining the depths from each step, the total depth required to compute $A_{i,j}$ is
    $$
    d_{\text{total}} = 4(d_{\mathrm{std}} + d_\oplus) +d_\triangle+ d_{\exp}.
    $$
    Since all entries $A_{i,j}$ for $i, j \in [n]$ can be computed in parallel, the overall circuit has size $\mathrm{poly}(n)$ and depth $4(d_{\mathrm{std}} + d_\oplus) + d_\triangle+ d_{\exp}$. Therefore, the attention matrix $A$ can be computed by a uniform threshold circuit with size $\mathrm{poly}(n)$ and depth $4(d_{\mathrm{std}} + d_\oplus) + d_\triangle + d_{\exp}$.

    Thus we complete the proof.
\end{proof}

Here we state the proof of Lemma~\ref{lem:mlp_tc0}.
\begin{lemma}[MLP computation in $\mathsf{TC}^0$, formal version of Lemma~\ref{lem:mlp_tc0}]\label{lem:mlp_tc0:informal}
    If $p \leq \poly(n)$, then
    the MLP layer in Definition~\ref{def:mlp} can be computable by a uniform threshold circuit with size $\poly(n)$ and depth $2d_\mathrm{std} + d_{\oplus}$.
\end{lemma}
\begin{proof}
    For each $i \in [m]$,
    by Lemma~\ref{lem:matrix_multi}, we need a circuit with depth $d_\mathrm{std} + d_{\oplus}$ and size $\poly(n)$ to compute $WX_{i, *}$, and by Part 1 of Lemma~\ref{lem:float_operations_TC}, w need a circuit with depth $d_\mathrm{std}$ and size $\poly(n)$ to compute $WX_{i, *} + b$. Hence the total depth need is $2d_\mathrm{std} + d_{\oplus}$ and total size is still $\poly(n)$. Since this procedure can be done in parallel for all $i \in [n]$, the proof is complete.
\end{proof}

Here we state the proof of Lemma~\ref{lem:layer_tc0}.
\begin{lemma}[Layer-norm computation in $\mathsf{TC}^0$, formal version of Lemma~\ref{lem:layer_tc0}]\label{lem:layer_tc0:informal}
    If $p \leq \poly(n)$, then
    the Layer-wise Normalization layer in Definition~\ref{def:layer_norm} can be computable by a uniform threshold circuit with size $\poly(n)$ and depth $5d_\mathrm{std} + 2d_{\oplus} + d_\mathrm{sqrt}$.
\end{lemma}
\begin{proof}
   For each $i \in [n]$, by Lemma~\ref{lem:float_operations_TC}, we can compute $\mu_i$ using a circuit with depth $d_\mathrm{std} + d_{\oplus}$ and size $\poly(n)$ and then compute $\sigma_i^2$ with depth $2d_\mathrm{std} + d_{\oplus}$ and size $\poly(n)$. By Lemma~\ref{lem:float_operations_TC} and Lemma~\ref{lem:sqrt}, we can compute $g^\mathrm{LN}(x)_{i,*}$ using a circuit with depth $2d_\mathrm{std} + d_\mathrm{sqrt}$ and size $\poly(n)$. Hence the total needed depth is $5d_\mathrm{std} + 2d_{\oplus} + d_\mathrm{sqrt}$ and size is $\poly(n)$. Since this procedure can be done in parallel for all $i \in [n]$, the proof is complete.
\end{proof}

Here we state the proof of Lemma~\ref{lem:tf}.
\begin{lemma}[Multi-layer $\rope$-based Transformer computation in $\TC^0$, formal version of Lemma~\ref{lem:tf}]\label{lem:tf:informal}
    Suppose that for each $i \in [m]$, $g_i$ in $\mathsf{TF}$ is computable by a constant depth $d_g$ uniform threshold circuit with size $\poly(n)$.
    If $p \leq \poly(n)$, then
    the $\rope$-based Transformer $\mathsf{TF}$ in Definition~\ref{def:multi_layer_self_attn} can be computable by a uniform threshold circuit with size $\poly(n)$ and depth $(m+1)d_g + 7m(d_\mathrm{std} + d_\oplus) + m(d_\triangle+ d_\mathrm{exp})$.
\end{lemma}
\begin{proof}
    For each $i \in [m]$, by condition, $g_i$ is computable by a constant depth $d_g$ uniform threshold circuit with size $\poly(n)$.
    
    For each $i \in [m]$, by Lemma~\ref{lem:attn}, $\mathsf{Attn}_i$ is computable by a uniform threshold circuit with depth $7(d_\mathrm{std} + d_\oplus)+d_{\triangle}+d_{\exp}$ and size $\poly(n)$. 

    Hence, to compute $\mathsf{TF}(X)$, we need to compute $g_0, g_1, \ldots, g_m$ and $\mathsf{Attn}_1, \ldots, \mathsf{Attn}_m$, thus the total depth of the circuit is $(m+1)d_g + 7m(d_\mathrm{std} + d_\oplus) + m(d_\triangle+ d_\mathrm{exp})$ and the size of circuit is $\poly(n)$.

    Thus we complete the proof.
\end{proof}

%%%% Cut-line between first 10 pages and appendix

%%% some writing rules

%% Writing rule for creating tags.
%% Tags :
%% Theorem    \ref{thm:bla_bla}
%% Lemma      \ref{lem:bla_bla}
%% Claim      \ref{cla:bla_bla}
%% Corollary  \ref{cor:bla_bla}
%% Fact       \ref{fac:bla_bla}
%% Definition \ref{def:bla_bla}
%% Section    \ref{sec:bla_bla}
%% Subsection \ref{sub:bla_bla}
%% Equation   \ref{eq:bla_bla}

\end{document}